\documentclass{article}
\usepackage{spconf,amsmath,amssymb,amsthm,mathtools,graphicx,hyperref}
\usepackage{xcolor}
\usepackage{acronym}
\usepackage{cleveref}
\usepackage[caption=false,font=footnotesize]{subfig}
\usepackage{cite}
\usepackage{balance}
    
\newcommand{\qtil}{\tilde q}
\newcommand{\Ltil}{\tilde L_\theta}

\newcommand{\E}{\mathbb E}
\newcommand{\R}{\mathbb R}

\acrodef{mdp}[MDP]{Markov decision process}
\acrodef{rl}[RL]{reinforcement learning}
\acrodef{td}[TD]{temporal-difference}

\theoremstyle{plain}
\newtheorem{theorem}{Theorem}

\newtheorem{assumption}{Assumption}      
\theoremstyle{remark}

\crefname{assumption}{assumption}{assumptions}
\Crefname{assumption}{Assumption}{Assumptions}

\title{A Contraction Framework for Stochastic Operators with Bootstrapping: Application to TD Learning}
\name{Ids van der Werf~$^{\dagger}$, Sergio Rozada~$^{\ddagger}$, and Antonio G. Marques~$^{\ddagger}$ 
}
\address{$^{\dagger}$Delft University of Technology, Dep. Microelectronics, Delft, The Netherlands \\  $^{\ddagger}$Rey Juan Carlos University, Dep. Signal Theory and Comms., Madrid, Spain}
\begin{document}
\ninept
\maketitle

\begin{abstract}
Many iterative algorithms rely on bootstrapping. A variable is updated using a second, frozen copy as a target, which is periodically replaced with the updated variable.
Majorize--minimize and inexact proximal-point methods share this structure, as does \ac{td} learning. However, existing convergence guarantees for scenarios that combine sampled updates with targets refreshed only every $K$ steps rely on the specific structure of the update, such as linear approximation or gradient-based inner steps, and on uniformly bounded sampling error. We instead model the sampled update as a stochastic operator on the parameter space, which reduces the analysis to a contraction argument that needs no gradient structure and allows the sampling error to grow with the iterates. Within this framework, we derive a finite-time bound for i.i.d.\ samples and any target-update period $K$. We show that the iterates converge geometrically in root mean square to a ball around the fixed point, provided the sensitivity to the frozen target is smaller than the contraction slack of the inner map. Existing deterministic frozen-target contraction and stochastic-gradient-type bounds follow as special cases of our framework, and simulations of \ac{td} learning reproduce the predicted contraction rate and scaling of the error floor with the step size.
\end{abstract}

\begin{keywords}
Reinforcement learning, target networks, semi-gradient methods, stochastic contraction
\end{keywords}
\section{Introduction}\label{sec:intro}

\Ac{rl} has become a key framework for decision-making problems in communications, signal processing, and control~\cite{arulkumaran2017deep, akiyama2024nonparametric, rozada2024tensor,thornton2020deep,dahrouj2021overview, bozkus2023link}. At its core lies the estimation of value functions, which quantify cumulative rewards and satisfy Bellman equations~\cite{bertsekas2019reinforcement}. 
In this context, \acf{td} learning~\cite{bertsekas2019reinforcement} has emerged as the workhorse approach for value-function estimation.
\ac{td} learning estimates value functions from sampled transitions by taking \emph{semi-gradient} steps that regress the current value estimate onto a \emph{bootstrapped target}, i.e., a target that itself depends on the value estimate. Because of this dependence, the update is not the gradient of any fixed objective and establishing convergence is challenging. Practical implementations mitigate the issue by holding the target fixed for a number of steps, the so-called \emph{target network}~\cite{mnih2015human}. In this work, we study the convergence of \ac{td} learning from a \textit{stochastic operator-theoretic perspective}. 

The convergence theory of \ac{td} learning originates from tabular methods such as Q-learning~\cite{watkins1992q}, which perform stochastic approximation \textit{directly} in the entries of value functions and whose convergence is well established~\cite{tsitsiklis1994asynchronous}. A large \ac{mdp}, however, requires \textit{parametric} value-function approximations~\cite{bertsekas1996neuro}, making the analysis substantially more involved. Classical results establish convergence for linear approximators when bootstrapped targets are computed using the current value estimate~\cite{tsitsiklis1997analysis, tadic2001convergence, bhandari2018finite}. In practice, however, nonlinear approximators are commonly used together with target networks. Convergence guarantees that simultaneously account for sampled transitions, targets refreshed only every $K>1$ steps, and approximators beyond the linear class remain scarce. Recently, an optimization-based approach has been  proposed to analyze \ac{td} learning without assuming linearity and for arbitrary target-update periods~\cite{asadi2023td}. However, this framework focuses on a deterministic variant of \ac{td} learning. Subsequent works have incorporated stochasticity or target networks, but each relies on a specific structure of the \ac{td} update (linear approximation, gradient-based inner steps) and on its own analytical device (Jacobian analysis~\cite{fellows2023target}, preconditioning~\cite{wu2025unifying}, switched linear systems~\cite{lee2026periodic}), so that their guarantees do not transfer to one another, let alone to inner updates that are not gradient steps. We relate these works to our results in \Cref{sec::related}, once the latter have been presented.

In this work, we resort to an operator framework, which is sufficiently flexible and powerful to model algorithms such as gradient descent~\cite{larsson2025unified}. We \textit{formulate \ac{td} learning from a general operator perspective} and characterize the conditions under which it converges with i.i.d.\ samples. Specifically, our contributions are: 

\vspace{.1cm}
\noindent
\hspace*{0.2cm}
\begin{minipage}{\dimexpr\linewidth-0.2cm\relax}
\noindent \textbf{C1)} Formulating the \ac{td} learning algorithm with parametric-value function approximation from an operator perspective.\\
\noindent \textbf{C2)} Establishing convergence in the stochastic setting with i.i.d.\ samples\footnotemark, without assuming linearity and for arbitrary target-update periods $K$. 
\end{minipage}
\footnotetext{In practice, one would ultimately like to consider Markovian noise, as naturally arises in MDPs. Extending the present results to this setting, together with developing sharper and more general proof techniques, exceeds the scope of a conference paper and is left for the journal version of this work.}
\vspace{.1cm}

In what follows, we first introduce the mathematical framework, establish the convergence result, and discuss its relation to existing work in technical detail. Afterward, we apply the framework to a simple \ac{td}$(0)$ learning problem. 

\section{Convergence results via operator theory}\label{sec:Operator_MainResult}

Many iterative algorithms maintain two coupled copies of the same variable and update them on different time scales. One copy is held fixed while the other is updated for a number of steps, after which the fixed copy is overwritten with the result and the process is repeated, as in majorize--minimize schemes~\cite{sun2017majorization}, inexact proximal-point methods~\cite{rockafellar1976monotone}, federated local updates~\cite{li2020federated} and \ac{td} learning, the case we analyze in detail in this work. In \ac{td} learning, the value-function parameter that defines the bootstrapped regression target is frozen while a second copy is fitted to that target for a few steps, a device known to stabilize learning in practice~\cite{mnih2015human}. Analyzing such schemes is challenging for two reasons.

\vspace{.1cm}
\noindent
\hspace*{0.2cm}
\begin{minipage}{\dimexpr\linewidth-0.2cm\relax}
\noindent\textit{a)} The inner updates rarely come from an explicit optimization problem: \ac{td} uses a semi-gradient rather than a gradient, and practical implementations incorporate ad-hoc modifications justified by their empirical performance. \\
\noindent\textit{b)} Every update is driven by samples, so the map being iterated is itself random. \end{minipage}
\vspace{.1cm}

\noindent We therefore work directly at the level of the \textit{update map, viewed as a stochastic operator on the parameter space}. Gradient, proximal-gradient or projected schemes are particular instances, but the operator need not be the gradient of anything. What matters is that it contracts the variable it updates, that it is stable with respect to the frozen one, and that its sampling error can be controlled in mean square. This section introduces the framework, particularizes it for \ac{td} target networks, states the convergence result, and gives its proof. 
 
\smallskip\noindent\textbf{Notation and scheme.}
Throughout, $\|\cdot\|$ is the Euclidean norm on $\R^d$ and, for matrices, the induced spectral norm. Let $\xi$ denote a random sample (in \ac{td}, a transition) with probability law $\mu$, and let $\hat T_\theta(w;\xi)$ be the \emph{sampled update map}. Given the frozen parameter $\theta\in\R^d$ and the sample $\xi$, it returns the next value of the working (input) variable $w\in\R^d$. The subscript identifies the copy that is held fixed during the inner steps, and the argument the copy being updated. At outer iteration $t$ the scheme freezes $\theta_t$, initializes $w_{t,0}=\theta_t$, performs $K$ inner steps
\begin{align}
    w_{t,k+1}=\hat T_{\theta_t}(w_{t,k};\xi_{t,k}),\qquad k=0,\dots,K-1,
    \label{eq:inner-hat}
\end{align}
and refreshes the target as $\theta_{t+1}=w_{t,K}$. We take \eqref{eq:inner-hat} as the primitive object and \emph{define} the population operator as its mean, $T_\theta(w)\!:=\!\E_\xi[\hat T_\theta(w;\xi)]$, and the \textit{sampling error} as the deviation from it, $\varepsilon(\theta,w,\xi)\!:=\!\hat T_\theta(w;\xi)-T_\theta(w)$, so that one inner step reads
\begin{align}
    w_{t,k+1} = T_{\theta_t}(w_{t,k}) + \varepsilon(\theta_t,w_{t,k},\xi_{t,k}).
    \label{eq:inner-rec}
\end{align}
Composing the $K$ inner steps defines the \emph{stochastic outer map} $\hat G(\theta_t):=w_{t,K}$, so that $\theta_{t+1}=\hat G(\theta_t)$, whose noiseless counterpart is the $K$-fold composition $G(\theta)=T_\theta^K(\theta)$. Note that the same $\theta_t$ enters every inner step, both as the parameter of the operator and as its initialization, hence $\hat G$ is not a composition of $K$ unrelated contractions~\cite{larsson2025unified}, and convergence is governed by $G$ rather than by $T_\theta$ alone. Finally, we write
\begin{align}
    H_{t,k}
    = \left(\theta_0,\ \xi_{0,0},\dots,\xi_{0,K-1},\ \dots,\ \xi_{t,0},
    \dots,\xi_{t,k-1}\right)
    \label{eq:history}
\end{align}
for the history preceding the draw of $\xi_{t,k}$, and $\E[\,\cdot \mid H_{t,k}]$ for the conditional expectation given that history. By construction, $\theta_t$ and $w_{t,k}$ are deterministic functions of $H_{t,k}$, whereas $\xi_{t,k}$ is drawn only after $H_{t,k}$ is realized.
 
\smallskip\noindent\textbf{Assumptions.}
Our analysis rests on one assumption on the population operator and one on the samples.
\begin{assumption}[Population operator]
\label{as:det}
There exist $q\in[0,1)$ and $L_\theta\ge 0$ such that, for all $w,w',\theta,\theta'\in\R^d$,
\begin{align}
    \|T_\theta(w) - T_\theta(w')\| & \leq q\,\|w-w'\|, \label{eq:A1}\\
    \|T_\theta(w) - T_{\theta'}(w)\| & \leq L_\theta\,\|\theta-\theta'\|. \label{eq:A2}
\end{align}
\end{assumption}
Condition \eqref{eq:A1} states that, with the target frozen, each inner step is a contraction, and \eqref{eq:A2} that moving the target moves the operator by a bounded amount. Applying \eqref{eq:A1} $K$ times and \eqref{eq:A2} once per inner step shows that $G(\theta)=T_\theta^K(\theta)$ is Lipschitz with constant $L_G(K)=q^K+L_\theta(1-q^K)/(1-q)$.
Since $1-L_G(K)=(1-q^K)(1-q-L_\theta)/(1-q)$, $G$ is a contraction for one (equivalently, every) $K\ge1$ if and only if $L_\theta<1-q$. In words, the sensitivity to the target must be smaller than the contraction slack of the inner solver. In that case $G$ has a unique fixed point $\theta^\star$. Moreover, $T_{\theta^\star}$ and $T_{\theta^\star}^K$ are contractions with the same unique fixed point, and $\theta^\star=T_{\theta^\star}^K(\theta^\star)$, so $T_{\theta^\star}(\theta^\star)=\theta^\star$; in particular, $\theta^\star$ does not depend on $K$.
\begin{assumption}[I.i.d.\ sampling with affine noise envelope]
\label{as:noise}
Each $\xi_{t,k}$ is drawn from $\mu$ independently of $H_{t,k}$, and there exist constants $\kappa_0$, $\kappa_w$, $\kappa_\theta\ge0$ such that, for all $\theta,w\in\R^d$,
\begin{align}
    \big(\E_\xi\|\varepsilon(\theta,w,\xi)\|^2\big)^{\!1/2}
    \!\leq\! \kappa_0 \!+\! \kappa_w\|w\| \!+\! \kappa_\theta\|\theta\|. 
    \label{eq:envelope-det}
\end{align}
\end{assumption}
Since $\theta_t$ and $w_{t,k}$ are functions of $H_{t,k}$ and $\xi_{t,k}$ is independent of it, Ass. \ref{as:noise} and the definition of $\varepsilon$ yield
\begin{align}
    \E\left[\varepsilon(\theta_t,w_{t,k},\xi_{t,k}) \mid H_{t,k}\right]
    &= 0, \label{eq:mds} \\
    \big(\E [\|\varepsilon(\theta_t,w_{t,k},\xi_{t,k})\|^2
    \!\mid \! H_{t,k} ] \big)^{\!1/2}
    \! & \leq \!\kappa_0 \!+\! \kappa_w\|w_{t,k}\| \! +  \!\kappa_\theta\|\theta_t\|,
    \label{eq:envelope}
\end{align}
almost surely, for all $t$ and $k$. Condition \eqref{eq:mds} states that the sampling errors form a martingale difference sequence with respect to $H_{t,k}$, which holds whenever the $\xi_{t,k}$ are drawn independently. Under Markovian sampling the conditional mean of $\varepsilon$ is a state-dependent bias, whose treatment we defer to the extended version. Regarding \eqref{eq:envelope}, an affine envelope, as opposed to a uniform variance bound, is essential (see \cite{larsson2025unified} for a related discussion). For \ac{td} the sampling error is affine in $(\theta,w)$ (\Cref{subsec:td}), so a uniform bound would presuppose bounded iterates, which is precisely what we want to prove. We keep $\kappa_w$ and $\kappa_\theta$ separate because they have different origins: $\kappa_w$ is the regression-type amplification present in any sampled fitting scheme, whereas $\kappa_\theta$ arises only through the bootstrapped target and, in \ac{td}, carries a factor $\gamma$. Finally, iterating \eqref{eq:envelope} shows that $\E\|w_{t,k}\|^2<\infty$ for all $(t,k)$ whenever $\E\|\theta_0\|^2<\infty$, which we assume throughout.
 
\subsection{Particularization to TD learning with a frozen target}
\label{subsec:td}
Let $\xi=(s,r,s')$ be a transition sampled from the stationary distribution of the policy, where $s$ is the current state, $r$ the current reward, and $s'$ the next state. Moreover, let $v(s;w)$ be the value-function approximator,  $\gamma\in(0,1)$ the discount factor, and $\delta(\theta,w,\xi):=r+\gamma v(s';\theta)-v(s;w)$ the so-called \ac{td} error \cite{sutton1988learning}. Then, if $\alpha>0$ is a step size, \ac{td} learning with a frozen target is \eqref{eq:inner-hat} with 
$$\hat T_\theta(w;\xi)=w+\alpha\,\delta(\theta,w,\xi)\,\nabla_w v(s;w),$$ 
i.e., a semi-gradient step on the squared \ac{td} error in which the target $\theta$ is not differentiated. By definition, $T_\theta(w)=w+\alpha\,\E_\xi[\delta(\theta,w,\xi)\nabla_w v(s;w)]$ and $\varepsilon$ is the centered version of the sampled step.
When $K=1$, this is standard \ac{td}$(0)$, since $w_{t,0}=\theta_t$. For $K>1$ the frozen $\theta_t$ plays the role of the target network of deep \ac{rl}~\cite{mnih2015human}, whose parameters are copied from the online network every $K$ updates.

For linear value-function approximation, i.e., $v(s;\theta)=\phi(s)^\top \theta$, the operator is affine, $T_\theta(w)=(I-\alpha A)w+\alpha(b+\gamma C\theta)$ with $A=\E[\phi(s)\phi^\top(s)]$, $C=\E[\phi(s)\phi(s')^\top]$, $b=\E[r\phi(s)]$. Hence Ass. \ref{as:det} and \ref{as:noise} hold with $q=\|I-\alpha A\|$, $L_\theta=\alpha\gamma\|C\|$, $\theta^\star=(A-\gamma C)^{-1}b$, where $q<1$ whenever $A\succ 0$ and $\alpha<2/\lambda_{\max}(A)$.
\begin{align}
    \kappa_0 & = \alpha\bigl(\E\|r\phi(s)-b\|^2\bigr)^{1/2}, \;
    \kappa_w = \alpha\bigl(\E\|\phi(s)\phi^\top(s)-A\|^2\bigr)^{1/2}, \notag \\
    \kappa_\theta & = \alpha\gamma\bigl(\E\|\phi(s)\phi^\top(s')-C\|^2\bigr)^{1/2}.
\end{align}
Each constant measures the spread of a sample quantity around the population quantity it estimates. For nonlinear $v$, Ass. \ref{as:det} is an assumption rather than a consequence. We do not claim it holds for arbitrary approximators.

\subsection{Main result}
We are now ready to present our main result, the contraction of the stochastic iterates $\theta_{t+1} = \hat{G}(\theta_t)$.

\begin{theorem}[Bound on stochastic iterates]
\label{thm:envfree_bound_iid}
Suppose Ass. \ref{as:det} and \ref{as:noise} hold, and define $\qtil :=\sqrt{q^2+\kappa_w^2}$ and $\Ltil := \sqrt{L_\theta^2+\kappa_\theta^2}$. If $\qtil+\Ltil<1$, the iterates $\theta_t$ satisfy
\begin{align}
    \left(\E\|\theta_t-\theta^\star\|^2\right)^{1/2}
    \leq \hat d + \rho^{\,t}\Bigl(
    \left(\E\|\theta_0-\theta^\star\|^2\right)^{1/2}-\hat d\Bigr)^{\!+}\!,
    \label{eq:envfree_bound_iid}
\end{align}
where $(x)^+\!:=\max(x,0)$ and the contraction factor and the error floor are, respectively,
\begin{align}
    \rho =\qtil^K + \Ltil\,\frac{1-\qtil^K}{1-\qtil},
    \;\;
    \hat d = \frac{\kappa_0'}{\sqrt{(1-\Ltil)^2-q^2}-\kappa_w},
    \label{eq:rho_dhat}
\end{align}
with $\kappa_0' := \kappa_0 + (\kappa_w+\kappa_\theta)\|\theta^\star\|$.
\end{theorem}
As announced, the theorem states that outer iterates converge geometrically, in root mean square, to a ball around the fixed point $\theta^*$. Note that the floor $\hat d$ does not depend on $K$, whereas $\rho$ decreases with $K$ towards $\Ltil/(1-\qtil)$. That is, additional inner steps accelerate each outer iteration but do not improve the steady-state accuracy. Particularizing the result for $K\!=\!1$ and $K\!\rightarrow\!\infty$ connects with classical analyses in the literature, as detailed in the next subsection. 
Finally, we note the condition $\qtil\!+\!\Ltil\!<\!1$ ensures that $\rho \!< \!1$ and $\hat d\! \geq \!0$.

\begin{proof}
Let $d_{t,k}:=(\E\|w_{t,k}-\theta^\star\|^2)^{1/2}$, so that $d_{t,0}$ and $d_{t,K}$ are the errors of $\theta_t$ and $\theta_{t+1}$, respectively, being the key values that we aim to bound. 
The proof has three steps. We first derive a one-step recursion $d_{t,k+1}\le g(d_{t,k})+\Ltil d_{t,0}$, in which $g$ is nonlinear. We then define $\hat d$ as the fixed point of the map $d \mapsto g(d)+\Ltil d$. 
Finally, we 
show that the error above the floor $(d_{t,k}-\hat d)^+$ contracts linearly, first along the inner and then across outer iterations. Throughout, Minkowski's inequality
is used in the form $(\E(X+Y)^2)^{1/2}\le(\E (X^2))^{1/2}+(\E (Y^2))^{1/2}$
for $X,Y\ge0$.

\noindent \emph{Step 1:} Writing $w_{t,k+1}-\theta^\star=\chi_k+\eta_k$ with $\chi_k:=T_{\theta_t}(w_{t,k})-T_{\theta^\star}(\theta^\star)$ and $\eta_k:=\varepsilon(\theta_t,w_{t,k},\xi_{t,k})$, and expanding the squared norm, we obtain
\begin{align}
    d_{t,k+1}^2 = \E\|\chi_k\|^2 + \E\|\eta_k\|^2 + 2\,\E[\chi_k^\top\eta_k].
    \label{eq:sq_expand}
\end{align}
Since $\chi_k$ is a deterministic function of $H_{t,k}$ and $\E[\eta_k \mid H_{t,k}] = 0$ by Ass. \ref{as:noise}, the tower rule gives 
%
    $\E\left[\chi_k^\top \eta_k\right]
    = \E\left[\chi_k^\top \E[\eta_k \mid H_{t,k}]\right] = 0.$
%
Hence \eqref{eq:sq_expand} simplifies to 
\begin{align}
    d_{t,k+1}^2 
    = \E\left[\|\chi_k\|^2\right]
    + \E\left[\|\eta_k\|^2\right].
    \label{eq:sq_clean}
\end{align}
For the first term, we write
$\chi_k = [T_{\theta_t}(w_{t,k}) - T_{\theta_t}(\theta^\star)]
        + [T_{\theta_t}(\theta^\star) - T_{\theta^\star}(\theta^\star)]$,
so the triangle inequality and Ass.~\ref{as:det} give
$\|\chi_k\| \le q\,\|w_{t,k}-\theta^\star\| + L_\theta\,\|\theta_t - \theta^\star\|$. 
Minkowski's inequality then yields
\begin{align}
    \E[\|\chi_k\|^2] 
    \leq \left(q\,d_{t,k} + L_\theta\,d_{t,0}\right)^2.
    \label{eq:det_sq_bound}
\end{align}
For the second term in \eqref{eq:sq_clean}, using the tower rule, the bound in \eqref{eq:envelope}, the inequalities $\|w_{t,k}\|\le\|w_{t,k}-\theta^\star\|+\|\theta^\star\|$ and $\|\theta_t\|\le\|\theta_t-\theta^\star\|+\|\theta^\star\|$, and finally Minkowski's inequality, gives
\begin{align}
    \E[\|\eta_k\|^2]
    &=\E\bigl[\E[\|\eta_k\|^2\mid H_{t,k}]\bigr]
    \le\E\bigl(\kappa_0+\kappa_w\|w_{t,k}\|+\kappa_\theta\|\theta_t\|\bigr)^2
    \notag\\
    &\le\E\bigl(\kappa_0'+\kappa_w\|w_{t,k}-\theta^\star\|
        +\kappa_\theta\|\theta_t-\theta^\star\|\bigr)^2\notag \\
    & \le(\kappa_0'+\kappa_w d_{t,k}+\kappa_\theta d_{t,0})^2 .
    \label{eq:noise_sq_bound}
\end{align}
Substituting \eqref{eq:det_sq_bound} and \eqref{eq:noise_sq_bound} into \eqref{eq:sq_clean} and taking square roots gives 
\begin{align}
    d_{t,k+1} 
    & \le \|[q d_{t,k}+L_\theta d_{t,0},\;
\kappa_w d_{t,k}+\kappa_\theta d_{t,0}+\kappa_0']^\top\|,
\end{align}
a Euclidean norm in $\R^2$. Writing the vector as $[q d_{t,k},\,\kappa_w d_{t,k}+\kappa_0']^\top+d_{t,0}[L_\theta,\kappa_\theta]^\top$ and applying the triangle inequality gives
\begin{align}
    d_{t,k+1}
    & \le g(d_{t,k})+\Ltil\,d_{t,0}, \label{eq:g_rec}
\end{align}
where $g(d) := \|[q d,\ \kappa_w d+\kappa_0']^\top\|=\sqrt{\qtil^2d^2+2\kappa_w\kappa_0'd+\kappa_0'^2}$.

\vspace{.1cm}
\noindent \emph{Step 2:}
The argument of the norm in $g$ is affine in $d$ with slope $[q,\kappa_w]^\top$, so it changes by $\qtil\,|d-d'|$ between $d'$ and $d$, and the reverse triangle inequality gives $|g(d)-g(d')|\le\qtil\,|d-d'|$. For $d\ge d'$ this yields $g(d)-g(d')\le\qtil\,(d-d')$; for $d<d'$ we have $g(d)-g(d')\le0$, as $g$ is non-decreasing on $[0,\infty)$ because $g^2$ is a polynomial with non-negative coefficients. Both cases combined give
\begin{align}
    g(d)-g(d')\le\qtil\,(d-d')^+ ,\qquad d,d'\ge0 .
    \label{eq:g_incr}
\end{align}
Next, observe that the error floor $\hat d$ provided in \eqref{eq:rho_dhat} is the non-negative solution of
\begin{align}
    g(\hat d)=(1-\Ltil)\,\hat d .
    \label{eq:dhat_def}
\end{align}
Indeed, squaring $g(\hat d)=(1-\Ltil)\hat d$ gives $[(1-\Ltil)^2-\qtil^2]\hat d^2-2\kappa_w\kappa_0'\hat d-\kappa_0'^2=0$. With $S:=\sqrt{(1-\Ltil)^2-q^2}$, real and larger than $\kappa_w$ because $\qtil<1-\Ltil$, the leading coefficient is $S^2-\kappa_w^2>0$, the discriminant is $4\kappa_0'^2S^2$, and the roots are $\kappa_0'/(S-\kappa_w)\ge0$ and $-\kappa_0'/(S+\kappa_w)\le0$. The non-negative one is the value in \eqref{eq:rho_dhat}.

\vspace{.1cm}
\noindent \emph{Step 3:}
Let $e_{t,k}:=(d_{t,k}-\hat d)^+$. Subtracting $\hat d=g(\hat d)+\Ltil\hat d$ from \eqref{eq:g_rec} and using \eqref{eq:g_incr} with $(d,d')=(d_{t,k},\hat d)$ yields
\[
d_{t,k+1}-\hat d
\le g(d_{t,k})-g(\hat d)+\Ltil\,(d_{t,0}-\hat d)
\le \qtil\,e_{t,k}+\Ltil\,e_{t,0}.
\]
The right-hand side is non-negative, so it also bounds
$e_{t,k+1}=\max\{d_{t,k+1}-\hat d,0\}$. Unrolling over $k=0,\dots,K-1$, gives
\begin{align}
    e_{t,K}\le\Bigl[\qtil^K+\Ltil\,\frac{1-\qtil^K}{1-\qtil}\Bigr]e_{t,0}
    =\rho\,e_{t,0}.
    \label{eq:e_rec}
\end{align}
Since $\theta_{t+1}=w_{t,K}$ means $e_{t+1,0}=e_{t,K}$, iterating \eqref{eq:e_rec} over
$t$ gives $e_{t,0}\le\rho^{\,t}e_{0,0}$. Finally $d_{t,0}-\hat d\le e_{t,0}$, so
$d_{t,0}\le\hat d+\rho^{\,t}e_{0,0}$, which is \eqref{eq:envfree_bound_iid}.
\end{proof}


\begin{figure*}[!t]
    \centering
    \subfloat
    {
    \includegraphics[width=0.32\textwidth]{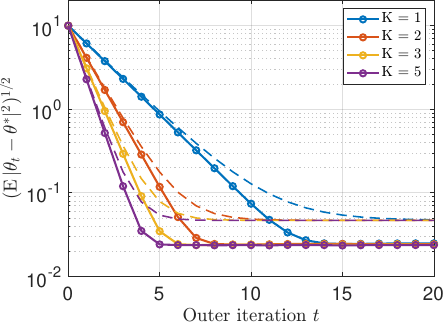}
    \label{subfig:1}
    }%
    \subfloat
    {
    \includegraphics[width=0.32\textwidth]{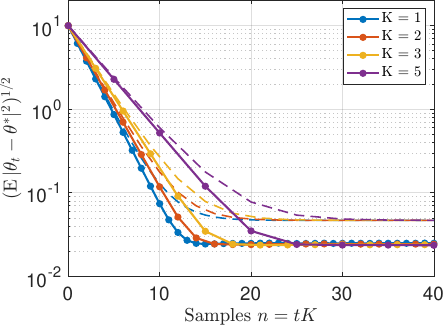}
    \label{subfig:2}
    }%
    \subfloat
    {
    \includegraphics[width=0.32\textwidth]{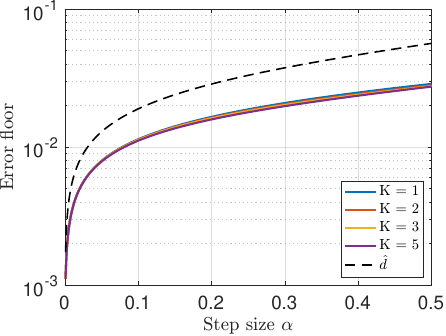}
    \label{subfig:3}
    }%
\caption{
Frozen-target \ac{td}$(0)$ on the two-state \ac{mdp}, for $K\in\{1,2,3,5\}$.
\textbf{Left} and \textbf{middle}: error $(\E\|\theta_t-\theta^\star\|^2)^{1/2}$ vs.\
outer iteration $t$ and vs.\ number of samples $n=tK$, averaged over $2\cdot10^4$ Monte
Carlo runs (solid), together with the bound of Thm.~\ref{thm:envfree_bound_iid}
(dashed).
\textbf{Right}: steady-state error vs.\ step size $\alpha$, computed exactly (solid) and
predicted by the floor $\hat d$ (dashed), which does not depend on $K$.
}
\label{fig:res}
\end{figure*}

\subsection{Discussion}
\label{sec::related}

Thm.~\ref{thm:envfree_bound_iid} establishes convergence guarantees for stochastic \ac{td} learning under i.i.d.\ sampling, arbitrary target-update periods, and general parametrizations satisfying our assumptions. The operator framework also provides explicit connections between existing convergence results.
First, the result connects \ac{td} learning with stochastic gradient descent analysis.
When $\Ltil=0$, we have $L_\theta=\kappa_\theta=0$, so $T_\theta(w)=T(w)$ and the noise envelope reduces to $\kappa_0+\kappa_w|w|$, removing the dependence on the target. The error floor then becomes $\hat d=\kappa_0'/(\sqrt{1-q^2}-\kappa_w)$, recovering the floor in~\cite[Eq.~(8)]{larsson2025unified}. The dependence on the target introduces the additional term $\Ltil(1-\qtil^K)/(1-\qtil)$ in $\rho$ and replaces $1$ by $(1-\Ltil)^2$ in the denominator of $\hat d$ in~\cite{larsson2025unified}.
Furthermore, in the deterministic limit $\kappa_0=\kappa_w=\kappa_\theta=0$, we obtain $\hat d=0$, $\qtil=q$, and $\Ltil=L_\theta$. Hence, the bound in \eqref{eq:envfree_bound_iid} reduces to $\|\theta_t-\theta^\star\|\le L_G(K)^t\|\theta_0-\theta^\star\|$. This recovers a contraction guarantee of the same nature as~\cite{asadi2023td}, which analyzes a deterministic version of frozen-target \ac{td} through iterative optimization beyond linear approximation.
We extend the analysis to i.i.d.\ stochastic updates using the operator-theoretic perspective.

For \ac{td} with linear approximation, $K=1$ gives $T_\theta(\theta)=\theta+\alpha[b-(A-\gamma C)\theta]$, the deterministic update underlying the classical finite-time analysis of~\cite{bhandari2018finite}. That work treats both i.i.d.\ and Markovian observations. We restrict sampling to i.i.d.\ observations but allow arbitrary $K$ and nonlinear approximators satisfying our assumptions.
For nonlinear \ac{td} with $K > 1$ under i.i.d.\ sampling, \cite[Thm.~3]{fellows2023target} bounds $\mathbb{E}[\|\theta_t-\theta^\star\|]$ using Jacobian-based arguments and uniformly bounded variance, typically requiring regularization. Our analysis controls the (stronger) root-mean-square error through $\rho$ and $\hat d$, allowing noise variance to grow with the iterates.

Under linear approximation, \cite{wu2025unifying} interprets \ac{td}, fitted Q-iteration (FQI), and partial FQI as differently preconditioned iterations of a linear system. In our notation, $G(\theta)=\theta+P_K[b-(A-\gamma C)\theta]$, where $P_K=\alpha\sum_{j=0}^{K-1}(I-\alpha A)^j$. Thus, $K=1$ gives \ac{td} with $P_1=\alpha I$, while $K\to\infty$ gives FQI with $P_K\to A^{-1}$ when $\|I-\alpha A\|<1$, corresponding to exact frozen-target regression. Their spectral analysis characterizes stability through $I-P_K(A-\gamma C)$, whereas our analysis bounds its norm by $L_G(K)$. Our operator formulation additionally accommodates nonlinear schemes.
Finally, $1-\rho=(1-\qtil^K)(1-\qtil-\Ltil)/(1-\qtil)$, so $\rho<1$ is equivalent to $\qtil+\Ltil<1$ for every $K\ge1$. Under this condition, $\rho$ is nonincreasing in $K$ and approaches $\Ltil/(1-\qtil)$ in the limit, while $\hat d$ is independent of $K$. Larger $K$ improves the contraction rate of the bound per target update, but need not improve sample efficiency (each update requires $K$ samples). This is further demonstrated in the next section. These properties concern our sufficient condition and error bound, not the actual stability region or steady-state error. In particular, they do not exclude stabilization through larger $K$, as established for deterministic linear approximation~\cite{lee2026periodic}.

\section{Simulations}\label{sec:sims}
We illustrate the framework on \ac{td} learning with linear function approximation on a two-state \ac{mdp}. The features are $\phi=(1,1.2)$, the rewards $r=(0,0.1)$, the discount factor $\gamma=0.2$, the stepsize $\alpha=0.4$ and the transition kernel is $P=\bigl(\begin{smallmatrix}0.95&0.05\\0.05&0.95\end{smallmatrix}\bigr)$, with stationary distribution $\pi=(\tfrac12,\tfrac12)$. Each inner step draws $s\sim\pi$ and $s'\sim P(s,\cdot)$ independently of the past, so Ass. \ref{as:noise} holds by construction. For these values, the hypothesis of Thm.~\ref{thm:envfree_bound_iid} is satisfied and the iterates converge to a ball around $\theta^\star$. We initialize at $\theta_0 = 10$ and average over $2\cdot 10^4$ Monte-Carlo runs.

Fig.~\ref{fig:res} (left) shows the root mean-square error against the outer iteration index $t$ for $K\in\{1,2,3,5\}$, together with the bound \eqref{eq:envfree_bound_iid}. The bound reproduces the observed decay rate for every $K$ and is conservative in the steady-state level, which is due to the worst-case alignment implicit in the envelope \eqref{eq:envelope}. The predicted contraction factors decrease with $K$ towards $\Ltil/(1-\qtil)$. Additional inner steps accelerate each outer iteration, but with rapidly diminishing returns.

Fig.~\ref{fig:res} (middle) plots the same trajectories against the number of samples consumed, $n=tK$. The order of the curves is now reversed, with larger $K$ slightly worse. The inner loop buys speed per outer iteration, not sample efficiency. The bound shows the same behavior.

Fig.~\ref{fig:res} (right) shows the exact steady-state error as a function of the step size $\alpha$, for each $K$, together with $\hat d$. The floor $\hat d$ correctly captures the dependence on $\alpha$, and not merely its value at a single operating point. For linear \ac{td}$(0)$, $1-\qtil-\Ltil$, $\kappa_0'$ and $\kappa_w$ are all $\Theta(\alpha)$, hence $\hat d=\Theta(\sqrt\alpha)$, matching the variance scaling of constant-step stochastic approximation.


\section{Conclusion}
We proposed a contraction framework for stochastic iterations with a frozen target and
used it to bound the root-mean-square error of the outer iterates under i.i.d.\
sampling, for any target-update period $K$ and without assuming gradient structure in
the inner map. Simulations of a simple linear \ac{td}$(0)$ problem show that the derived contraction
factor tracks the observed decay and that both improve with $K$ per outer iteration. Using larger $K$, however, is not necessarily sample efficient. Furthermore, the error floor scales as $\Theta(\sqrt{\alpha})$, matching the variance of constant-step stochastic approximation.


\newpage
\balance

\section{Acknowledgments}
This work was supported by the Spanish AEI \newline (AEI/10.13039/501100011033) grants PID2022-136887NB-I00 and PID2025-170000NB-I00, and the Community of Madrid via IDEA-CM (TEC-2024/COM-89), and the Ellis Madrid Unit. Beyond this support, the authors have no relevant financial or non-financial interests to disclose. Claude AI was used to assist in writing the manuscript and coding the simulations. The authors take full responsibility for the results of this paper.

\section{Compliance with Ethical Standards}
This is a numerical simulation study for which no ethical approval was required.

\bibliographystyle{IEEEbib}
\bibliography{references}

\end{document}